\documentclass[twoside,11pt]{article}

\usepackage{jmlr2e}
\usepackage{amsmath,xcolor,comment}

\ShortHeadings{On Convex Clustering Solutions}{Canh Hao Nguyen and Hiroshi Mamitsuka}
\firstpageno{1}

\begin{document}

\title{On Convex Clustering Solutions}

\author{\name Canh Hao Nguyen \email canhhao@kuicr.kyoto-u.ac.jp \\
	\name Hiroshi Mamitsuka \email mami@kuicr.kyoto-u.ac.jp \\
       \addr Bioinformatics Center, ICR, Kyoto University \\
       Gokasho, Uji, Kyoto, 611-0011, Japan}
       
\editor{--}

\maketitle

\begin{abstract}
Convex clustering is an attractive clustering algorithm with favorable properties such as efficiency and optimality owing to its convex formulation. It is thought to generalize both k-means clustering and agglomerative clustering. However, it is not known whether convex clustering preserves desirable properties of these algorithms. A common expectation is that convex clustering may learn difficult cluster types such as non-convex ones. Current understanding of convex clustering  is limited to only  consistency results on well-separated clusters.  We show new understanding of its solutions. We prove that convex clustering  can only learn convex clusters. We then show that the clusters have disjoint bounding balls with significant gaps. We further characterize the solutions, regularization hyperparameters, inclusterable cases and consistency. 
\end{abstract}

\begin{keywords}
  Convex clustering, K-means clustering, agglomerative clustering, bounding balls.
\end{keywords}

\section{Introduction}
Clustering is a challenging  problem and so far has not been well understood. Due to its usual non-convex objective functions, clustering algorithms are usually either inefficient, non-optimal or unstable. Not knowing the number of cluster beforehand is a major problem of almost all methods. Convex clustering \citep{Pelckmans05,Hocking11,Lindsten11} holds much promise as it offers optimal solutions to clustering without the need to specify the number of clusters beforehand. Its convex formulation offers computational efficiency and optimality of the solutions. Convex clustering  is thought to generalize both k-means clustering \citep{macqueen67} and agglomerative clustering \citep{johnson67} due to its \emph{formulation}. However, it is not clear on how convex clustering relates to k-means and agglomerative clusterings in terms of \emph{solutions}.

In this paper, we consider the basic nonweighted version of convex clustering formulated as follows. Given a data set in matrix form $X = [x_1,x_2,\cdots x_n], X \in \mathbb{R}^{d\times n}$ with column vector $x_i \in \mathbb{R}^d$, convex clustering  finds $\bar{U} = [\bar{u}_1, \bar{u}_2 \cdots \bar{u}_n ]  \in \mathbb{R}^{d\times n}$ with column vector $\bar{u}_i \in \mathbb{R}^d$ satisfying the following equation (with variables $u_i$): 
\begin{align}\label{maineq}
\bar{U} &= \arg \min_{u_i \in \mathbb{R}^{d}}  \frac{1}{2} \sum_{i=1}^n \|x_i-u_i\|^2 + \lambda \sum_{j<i} \|u_i-u_j\|.
\end{align}
$\bar{u}_i$ is understood to be the prototype (centroid) of $x_i$ and $\bar{u}_i=\bar{u}_j$ means that $x_i$ and $x_j$ belong to the same cluster.  Squared loss $\|u_i-x_i\|^2$ resembles that of k-means clustering.  Fusion penalty  $\|u_i- u_j\|$ is a convex relaxation of $l_0$ that requires many $\bar{u}_i=\bar{u}_j$, producing a small number of clusters.  By increasing hyperparameter $\lambda$ ($\in \mathbb{R}_+$), intuitively  $\sum_{i,j} \|u_i-u_j\|$ will be smaller, encouraging more $\|u_i- u_j\| = 0$.  This allows convex clustering to produce different numbers of clusters optimally without specifying the exact number of desired clusters, in the same way as agglomerative clustering \citep{johnson67}. Fusion penalty is a sparse formulation \citep{Hastie15}, which is well studied in many other tasks  such as total variation denoising \citep{Rudin92}, fused lasso \citep{Tibshirani05}, network lasso \citep{Hallac15},  trend filtering \citep{Wang15} and sparse hypergraphs \citep{Nguyen20}.

However, understanding of convex clustering is still very limited. An easy version of the original problem is using $l_1$ fusion penalty function ($\sum_{j<i} \|u_i-u_j\|_1$), which results in separating each dimension of the space into a different problem \citep{Hocking11,Radchenko17}. Recent advances are mainly on its variations \citep{Tan15,Wang16,Shah17,Wang18}, properties of weights in its weighted variation \citep{Sun18, Chi19} and computational efficiency of its optimization algorithms \citep{Chi15,Panahi17,Yuan18}. These results do not bring any new understanding of its solutions.

In terms of properties of convex clustering's solutions, previous work showed that it is able to recover \emph{well-separated clusters} such as cube clusters \citep{Zhu14} with significant distance between the cubes, or  more general shapes \citep{Panahi17, Yuan18}. Parts of Gaussian components in a mixture containing the points lying within some fixed number of standard-deviations for each mean of Gaussian mixtures can also be recovered \citep{Jiang20}.  Well-separated clusters might be easily learnt by many algorithms. This is  just a sufficient condition, specifying some special cases with guaranteed  solutions. However, convex clustering produces clusters in any case.  As its formulation is regarded as a relaxation of k-means and agglomerative clustering, it is not clear what type of clusters convex clustering can learn in general and how similar the solutions are to those of k-means (Voronoi cells) or agglomerative clustering (potentially nonconvex clusters).

In this paper,  we prove  important properties of the solutions of convex clustering, contrasting it from other algorithms. Contrary to common expectation, we show that convex clustering  can \emph{only learn convex clusters}, unlike agglomerative clustering. We further show that the clusters can be bounded by \emph{disjoint bounding balls with radii depending on their sizes}. We can say that convex clustering produces  \emph{circular clusters}.  Importantly, there are always \emph{significant gaps} among the  bounding balls. This shows a fundamental difference of convex clustering from k-means and other partition-based clustering algorithms that fill up the space. We further show general characteristics on: 1) the samples that result in the same solutions by convex clustering, 2) intuitive guidelines on hyperparameter setting, 3) a case that is impossible to cluster and 4) a guideline to achieve statistical consistency. 

We show the proof of convex clusters in Section 2. In Section 3, we show the clusters' bounding balls, their sizes and gap. In Section 4, we further show general characteristics of convex clustering. We carried out experiments to demonstrate properties of convex clustering solutions more intuitively in Section 5. We then summarize our findings and discuss future work.

\section{Convex Clustering  Learns Convex Clusters}

 \textbf{Notations.} Let $x = [x_1^T,x_2^T,\cdots x_n^T]^T$, $u =  [u_1^T,u_2^T,\cdots u_n^T]^T$ and $\bar{u} = [\bar{u}_1^T,\bar{u}_2^T,\cdots \bar{u}_n^T]^T$ ($\in \mathbb{R}^{nd}$) be the vectors formed by stacking $n$ components  $x_i$ $u_i$ and $\bar{u_i}$ respectively. Let $f$ denote the objective function of (\ref{maineq}): 
\begin{equation}\label{objective}
f(u) =  \frac{1}{2} \sum_{i=1}^n \|u_i-x_i\|^2 + \lambda \sum_{j<i} \|u_i-u_j\|_2,
\end{equation} 
and $\bar{u} = \arg\min_{u} f(u)$. Suppose that the solution set $\{\bar{u}_i\}_{i=1}^n$  contains $k$ distinct vectors $m_l$, $l = 1\cdots k$ that every $\bar{u}_i = m_l$ for some $l = 1\cdots k$. Let $k$ subsets of the training sets $V_1,V_2\cdots V_k$ denote the clustering partitions of the data that $x_i \in V_l$ if and only if $\bar{u}_i = m_l$, with $\|V_l\| =n_l$ as its cardinality.  

\begin{theorem}
\textbf{(Cluster convexity).} All clusters discovered by convex clustering are convex in the sense that the interiors of their convex hulls are disjoint. Let $H_l$ be the convex hull of the cluster $l$ defined on its set of points $x_i \in V_l$, $\forall 0 < l \leq k \in \mathbb{N}$. For any $o \neq l$, $int(H_l)  \cap int(H_o) = \emptyset$. 
\end{theorem}
\begin{proof}
\begin{figure}
  \centering
  \includegraphics[width=0.5\linewidth]{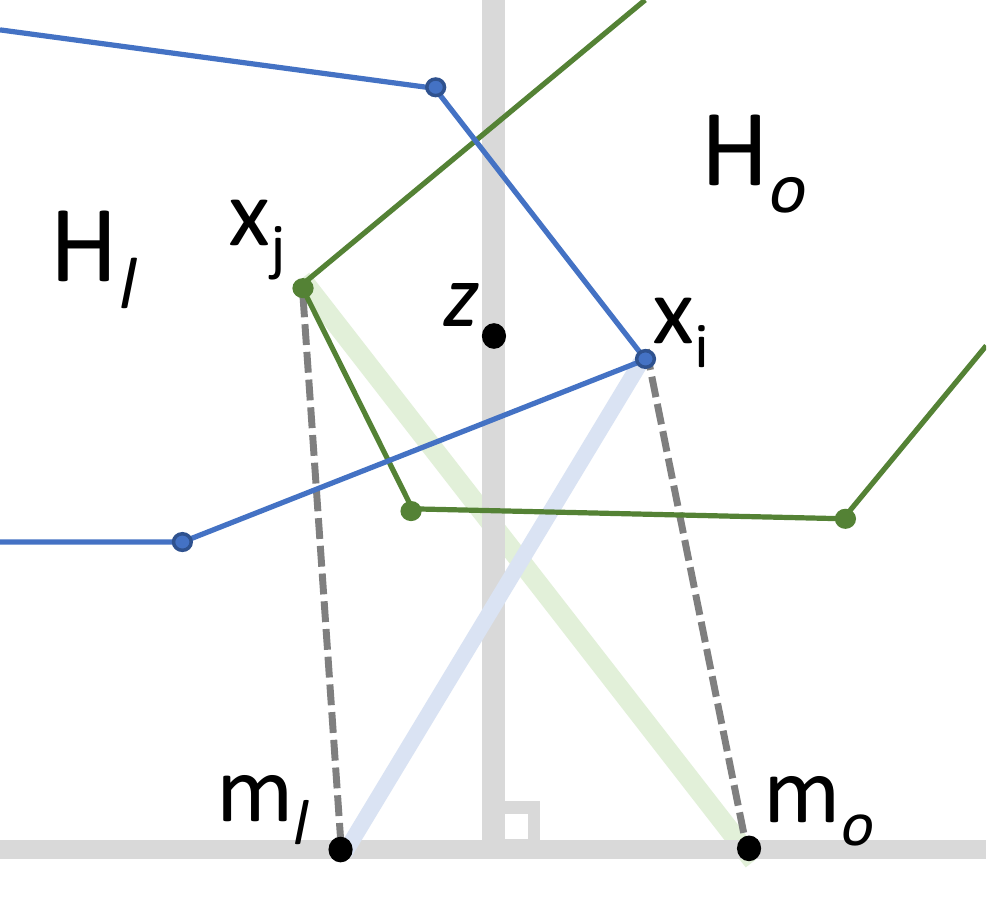}
  \caption{Idea of convexity proof. If  two convex hulls of two clusters intersect, there exist two samples ($x_i$ and $x_j$) to be swapped (dashed lines) to obtain a smaller objective function.}
  \label{fig:sub1}
\end{figure}%

We show the idea of proof here in Figure \ref{fig:sub1}. We prove by contradiction. If there exist two convex hulls that intersect, then there exist two samples of which the prototypes can be swapped to obtain a smaller objective function than the optimal one (by having the same regularization and smaller squared loss). 

Suppose that there is a non-empty intersection of two convex hulls $H_l$ and $H_o$, we can find a $z$ in the interiors of both $H_l$ and $H_o$. For any vector $a \in \mathbb{R}^d, a \neq 0$, we prove that there exists a vertex $x_i \in H_l$ that  $\langle a,x_i\rangle  > \langle a, z\rangle$. We consider a linear  transformation $g: \mathbb{R}^d \rightarrow \mathbb{R},\ , \forall v \in \mathbb{R}^d,   g(v) = \langle a,v\rangle $ and then $g(H_l) \subset \mathbb{R}$ will be a line segment as $g$ is continuous. The level curves of $g$ will be hyperplanes orthogonal to $a$, and therefore, the maximum of  $g(H_l)$ must be attained at a vertex of $H_l$, which we call $x_i$. As $z$ is in the interior of $H_l$, $g(z)$ is in the interior of $g(H_l)$, therefore 
\begin{equation}\label{theo21} 
\langle a,x_i\rangle  > \langle a, z\rangle.
\end{equation}

In a similar manner, there  exists a vertex $x_j \in H_o$ that 
\begin{equation}\label{theo22} 
\langle -a,x_j\rangle  > \langle -a,z\rangle.
\end{equation} 
Combining  (\ref{theo21}) and (\ref{theo22}), for any fixed $a$, we can have $\langle a,x_i\rangle  > \langle a,z\rangle  > \langle a,x_j\rangle $.
Consider the case that $a = m_o - m_l$, then 
\begin{align}\label{theorem23}
\langle m_o - m_l,x_i\rangle  &> \langle m_o-m_l,x_j\rangle   \nonumber \\
-\langle m_o,x_j\rangle -\langle m_l,x_i\rangle &> - \langle m_o,x_i\rangle - \langle m_l,x_j\rangle    \nonumber \\
\|x_i-m_l\|^2 + \|x_j - m_o\|^2 &> \|x_i - m_o\|^2 + \|x_j - m_l\|^2.
\end{align}

We now show that there is another input $y$ of $f$ that makes  $f(y) < f(\bar{u}$). We form $y$ from $\bar{u}$ by swapping $\bar{u}_i (= m_l)$ to $\bar{u}_j  (= m_o)$ to get a smaller squared loss while the  regularization part remains the same, i.e.  $y =  [\bar{u}_1^T, \cdots \bar{u}_{i-1}^T, \bar{u}_{j}^T, \bar{u}_{i+1}^T, \cdots \bar{u}_{j-1}^T, \bar{u}_{i}^T, \bar{u}_{j+1}^T, \cdots \bar{u}_n]^T$. 

Consider $f(y) - f(\bar{u})$. Notice that  the regularization parts, involving all pairwise distances (with different orderings), are the same for $f(y)$ and $f(\bar{u})$. Hence, the regularization parts cancel out each other in  $f(y) - f(\bar{u})$. 

The squared loss parts, except for those involving $\bar{u}_i$ and $\bar{u}_j$, are the same for all terms, also cancel out each other in $f(y) - f(\bar{x})$. Then, only the the terms involving $\bar{u}_i$ and $\bar{u}_j$ remain: 
\begin{align}
f(y) - f(\bar{u}) =  \frac{1}{2} \left( \|x_i - \bar{u}_j\|^2 + \|x_j - \bar{u}_i\|^2  - \|x_i - \bar{u}_i\|^2 - \|x_j - \bar{u}_j\|^2 \right) < 0,
\end{align}
according to (\ref{theorem23}). This is contradictory to the optimality of $\bar{u}$. That is, if any two of the convex hulls of the clusters overlap, then there will be two samples from two clusters that we can swap their prototypes to arrive at a solution that has even smaller loss function than the optimal one. Hence, it is concluded that interiors of the convex hulls of clusters must be disjoint, i.e., the clusters are convex.
\end{proof}

This shows the key difference to agglomerative clustering, which can produce nonconvex clusters.  


\section{Cluster Bounding Balls}
In this section, we show the main result that the clusters learnt by convex clustering can be bounded by disjoint bounding balls with significant gaps (Theorem \ref{theoboundingball}) and other properties.

\textbf{Notation.} Let $\epsilon \in \mathbb{R}^{nd}$, $\epsilon= [\epsilon_1^T, \epsilon_2^T, \cdots \epsilon_n^T]^T, \epsilon_i \in \mathbb{R}^d, \|\epsilon\| = 1$ be any unit vector in $\mathbb{R}^{nd}$.  In this section, we will take partial derivative in the direction of $\epsilon$ to elucidate the consequences of the optimality condition of (\ref{maineq}) and arrive at desirable results. Let $e_{ij}$ denote the unit vector from $\bar{u}_j$ to $\bar{u}_i$ for $\bar{u}_i \neq \bar{u}_j$: $ e_{ij} = \frac{\bar{u}_i-\bar{u}_j}{\|\bar{u}_i-\bar{u}_j\|}$. Then $\|e_{ij}\| = 1$ and $e_{ij} = - e_{ji}$. Let $E_i = \sum_{j | \bar{u}_i \neq \bar{u}_j} e_{ij}$. Note that if $\bar{u}_i = \bar{u}_j$ then $E_i = E_j$, meaning that  $E_i$ is the same for all samples in a cluster.

\subsection{Optimality Condition}
\textbf{Observation.}  Even though $f$ is not differentiable, it is \emph{directionally differentiable} because all of its components, i.e. a squared loss and fusion penalty (pairwise distances), are. The key observation of optimality condition for (\ref{maineq}) is that, at the solution of the problem (\ref{maineq}), \emph{all of its directional derivatives are nonnegative}. We take the directional derivative of $f$ (\ref{objective}) at $\bar{u}$ in the direction of $\epsilon, \forall \epsilon \in \mathbb{R}^{nd}, \|\epsilon\| = 1$, for $0 < \gamma \in \mathbb{R}$, $u = \bar{u} + \gamma \epsilon$:
\begin{equation}\label{eq:cond1}
\partial^{\epsilon}_{\bar{u}} (f) = \lim_{\gamma \rightarrow 0} \frac{f(u)-f(\bar{u})}{\gamma } \geq 0.
\end{equation}

We first derive the general formula for $\partial^{\epsilon}_{\bar{u}} (f)$ before considering its special cases of interest.
\begin{lemma}\label{lemma1}
General formula for directional derivative. 
\begin{align}\label{eqlemma1}
\partial^{\epsilon}_{\bar{u}} (f) = \sum_i \langle \bar{u}_i - x_i + \lambda E_i, \epsilon_i \rangle + \lambda \sum_{j < i | \bar{u}_i = \bar{u}_j} \|\epsilon_i-\epsilon_j\|.
\end{align}
\end{lemma}

\begin{proof}

Let's unfold the  directional derivative. 
\begin{align}
& f(u)-f(\bar{u}) \nonumber \\
=  &\frac{1}{2}  \sum_{i = 1}^n  (x_i-\bar{u}_i-\gamma \epsilon_i)^2  + \lambda \sum_{j<i}  \|\bar{u}_i + \gamma \epsilon_i -\bar{u}_j - \gamma \epsilon_j\| - \frac{1}{2}  \sum_{i = 1}^n (x_i-\bar{u}_i)^2   -  \lambda \sum_{j < i}  \|\bar{u}_i-\bar{u}_j\|  \nonumber \\
= &\sum_{i = 1}^n(\langle\bar{u}_i-x_i,\gamma \epsilon_i\rangle  + \frac{1}{2}\gamma^2 \|\epsilon_i\|^2) + \lambda \sum_{j < i}  ( \|\bar{u}_i + \gamma \epsilon_i -\bar{u}_j - \gamma \epsilon_j\|  - \|\bar{u}_i-\bar{u}_j\|).
\end{align}

Denote $f_1(u) \stackrel{\text{def}}{=} \sum_{i = 1}^n(\langle\bar{u}_i-x_i,\gamma \epsilon_i\rangle  + \frac{1}{2}\gamma^2\|\epsilon_i\|^2)$ for the squared loss part, and $f_2(u) \stackrel{\text{def}}{=} \lambda \sum_{j < i}  ( \|\bar{u}_i + \gamma \epsilon_i -\bar{u}_j - \gamma \epsilon_j\|  - \|\bar{u}_i-\bar{u}_j\|)$ for the regularization part.  $f(u)-f(\bar{u}) = f_1(u) + f_2(u)$.  Then,

\begin{align}\label{1part1}
\lim_{\gamma \rightarrow 0} \frac{f_1(u)}{\gamma } = \sum_{i = 1}^n(\langle\bar{u}_i-x_i, \epsilon_i\rangle.
\end{align}

To compute $\lim_{\gamma \rightarrow 0} \frac{f_2(u)}{\gamma } $, consider each component of the sum. There are two cases. 

Case 1, within-cluster fusion penalty: $\bar{u}_i = \bar{u}_j$, then 
\begin{equation}\label{1case1}
\lim_{\gamma \rightarrow 0} \frac{\|\bar{u}_i + \gamma \epsilon_i -\bar{u}_j - \gamma \epsilon_j\|  - \|\bar{u}_i-\bar{u}_j\|}{\gamma }  =  \|\epsilon_i-\epsilon_j\|.
\end{equation}

Case 2, between-cluster fusion penalty: $\bar{u}_i \neq \bar{u}_j$, first 
\begin{align}
& \|\bar{u}_i +\gamma\epsilon_i - \bar{u}_j - \gamma\epsilon_j\| -   \|\bar{u}_i-\bar{u}_j\| \nonumber \\
= &  \frac{\|\bar{u}_i +\gamma\epsilon_i - \bar{u}_j - \gamma\epsilon_j\|^2 - \|\bar{u}_i-\bar{u}_j\|^2 }{\|\bar{u}_i +\gamma\epsilon_i - \bar{u}_j - \gamma\epsilon_j\| + \|\bar{u}_i-\bar{u}_j\|} \nonumber \\
=  & \frac{2\langle\bar{u}_i-\bar{u}_j,\gamma (\epsilon_i -\epsilon_j)\rangle  + \gamma^2\|\epsilon_i-\epsilon_j\|^2}{\|\bar{u}_i + \gamma \epsilon_i - \bar{u}_j\| +   \|\bar{u}_i-\bar{u}_j\|}.
\end{align}
Then,
\begin{align}\label{1case2}
& \lim_{\gamma \rightarrow 0} \frac{  \|\bar{u}_i +\gamma\epsilon_i - \bar{u}_j - \gamma\epsilon_j\| -   \|\bar{u}_i-\bar{u}_j\|}{\gamma } \nonumber\\
= &\lim_{\gamma \rightarrow 0} \frac{2\langle\bar{u}_i-\bar{u}_j,\gamma (\epsilon_i -\epsilon_j)\rangle  }{2\gamma \|\bar{u}_i-\bar{u}_j\|} + \frac{\gamma \|\epsilon_i-\epsilon_j\|^2}{2\|\bar{u}_i-\bar{u}_j\|}  \nonumber\\
= &\langle e_{ij}, (\epsilon_i-\epsilon_j)\rangle \nonumber \\
= &\langle e_{ij}, \epsilon_i\rangle  + \langle e_{ji}, \epsilon_j\rangle.
\end{align}

Then from (\ref{1case1}) and (\ref{1case2}),
\begin{align}\label{1part2}
\lim_{\gamma \rightarrow 0} \frac{f_2(u)}{\gamma }  = &\lambda \sum_{j < i | \bar{u}_i = \bar{u}_j} \|\epsilon_i-\epsilon_j\|  +  \lambda \sum_{j < i | \bar{u}_i \neq \bar{u}_j} \langle e_{ij}, \epsilon_i\rangle  + \langle e_{ji}, \epsilon_j\rangle  \nonumber \\
 = & \lambda \sum_{i, j | \bar{u}_i \neq \bar{u}_j} \langle e_{ij}, \epsilon_i\rangle  \nonumber \\
 = & \lambda \sum_{i} \langle E_i, \epsilon_i\rangle.
\end{align}

From (\ref{1part1}) and (\ref{1part2}), we have the directional derivative $\partial^{\epsilon}_{\bar{u}} (f) $
\begin{align}
=  &\sum_i \langle \bar{u}_i - x_i, \epsilon_i\rangle  + \lambda \sum_{j < i | \bar{u}_i = \bar{u}_j} \|\epsilon_i-\epsilon_j\|  +  \lambda \sum_{i} \langle E_i, \epsilon_i\rangle   \nonumber \\
= &\sum_i \langle \bar{u}_i - x_i +\lambda E_i,  \epsilon_i\rangle  + \lambda \sum_{j < i | \bar{u}_i = \bar{u}_j} \|\epsilon_i-\epsilon_j\| .
\end{align}
\end{proof}
This is  the directional derivative in the direction of $\epsilon$, which should be nonnegative $\forall \epsilon$. Using this lemma, we consider the directional derivatives at some specific $\epsilon$ of interest to arrive at the following main result. 

\subsection{Bounding Balls}
The main result is summarized in the next theorem, followed by supplementary result on the tightness of the ball.
\begin{theorem}\label{theoboundingball}
\textbf{(Cluster bounding balls.)} There exist $k$ bounding balls $C_l \subset R^d, C_l = \{v \in \mathbb{R}^d \| \|v-c_l \| \leq r_l\}$, $l = 1\cdots k$ corresponding to $k$ clusters $V_l$,  centered at $c_l \in R^d$  with corresponding radii $r_l \in R$ that   
\begin{enumerate}
\item Each bounding ball covers its corresponding cluster in the sense that  $\forall x_i \in V_l$ then $x_i \in C_l$ ($\|u_i-c_l \| \leq r_l$). Its center and radius are: $c_l = \lambda E_i + m_l$ and $r_l = \lambda(n_l - 1)$ respectively.
\item The ball centers are the means of corresponding clusters: $c_l = \frac{\sum_{\bar{u}_i = m_l} x_i}{n_l}$.    
\item The balls are disjoint, separated by at least $2\lambda$. 
\end{enumerate}
\end{theorem}
The theorem shows that the clusters learnt by convex clustering is different from those of k-means (Voronoi cells). They can be bounded by disjoint balls, which  do not fill up the space. The distances from  boundaries to the means of the clusters depend on the number of samples in the clusters, which are different from k-means (having equal distances to the closest cluster centers). Importantly, there are significant gaps among the bounding balls. 

\begin{proof}

\textbf{Part 1.} We choose $\epsilon = (0^T,0^T \cdots 0^T, \epsilon_i^T, 0^T \cdots 0^T)^T$ to take directional derivative in this case.  That is $\epsilon$ is only nonzero at its $i$ component, $\|\epsilon_i\| = 1$ and $\epsilon_j = 0 \forall j \neq i$. Let $x_i \in V_l$ (in the $l$ cluster).  Let $u = \bar{u} + \epsilon$. From lemma \ref{lemma1}, we compute its directional derivative:
\begin{align}
\partial^{\epsilon}_{\bar{u}} (f)  
&= \langle\bar{u}_i -x_i + \lambda E_i, \epsilon_i \rangle + \lambda \sum_{j | \bar{u}_j = \bar{u}_i} {\|\epsilon_i - \epsilon_j\|} \nonumber \\
&= \langle\bar{u}_i -x_i + \lambda E_i, \epsilon_i \rangle + \lambda (n_l - 1).  
\end{align}
From (\ref{eq:cond1}), choosing $\epsilon_i$ in the direction of $ -(\bar{u}_i -x_i +\lambda E_i) $ gives 
\begin{align}
 & -\|\bar{u}_i -x_i +\lambda E_i \| + \lambda (n_l - 1) \geq 0 , \nonumber \\
 &\lambda (n_l - 1) \geq \| (\bar{u}_i  +\lambda E_i) -x_i \|. 
\end{align}
This shows that $x_i$ lies within a distance $\lambda (n_l - 1)$ from $\bar{u}_i  +\lambda E_i $ ($= \lambda E_i + m_l$), e.g.  $x_i$ is contained in $C_l$, the ball with radius and center:
\begin{align}
r_l &= \lambda (n_l-1), \label{radius} \\
c_l &=  \lambda E_i + m_l.  \label{center}
\end{align}
 Note that $E_i$, due to the way it is defined,  is the same for all samples in the cluster, meaning that $c_l$ is common $\forall x_i \in V_l$.

\textbf{Part 2.} Consider a fixed cluster $l$, let $\epsilon_i = \epsilon' \forall  \bar{x_i} = m_l$ with $\|\epsilon'\| = \frac{1}{\sqrt{n_l}}$, $\epsilon_j = 0 \forall  \bar{u}_j \neq m_l$  (to make $\epsilon$ a unit vector). That is, $\epsilon$ is only nonzero at the components corresponding to cluster $l$, and all these components are equal to each other ($\epsilon'$). From lemma \ref{lemma1}, we compute its directional derivative:
\begin{align}
\partial^{\epsilon}_{\bar{u}} (f) &= \sum_i \langle  \bar{u}_i - x_i + \lambda E_i, \epsilon_i \rangle  + \underbrace{\lambda \sum_{j < i | \bar{u}_i = \bar{u}_j} \|\epsilon_i-\epsilon_j\|}_{=0} \nonumber \\
& =  \sum_{i| \bar{u}_i = m_l}  \langle \bar{u}_i - x_i + \lambda E_i, \epsilon' \rangle 
\end{align}

From (\ref{eq:cond1}), we have $ \sum_{i| \bar{u}_i = m_l} ( \bar{u}_i - x_i + \lambda E_i) = 0$ (by choosing $\epsilon'$ in the direction of $-\sum_{i |\bar{u}_i = m_l} ( \bar{u}_i - x_i + \lambda E_i)$ resulting in $-\|\sum_{i |\bar{u}_i = m_l} ( \bar{u}_i - x_i + \lambda E_i)\| \geq 0$, implying $\|\sum_{i |\bar{u}_i = m_l} ( \bar{u}_i - x_i + \lambda E_i)\| = 0$). 
Therefore,
\begin{align}\label{centermean}
 \sum_{i | \bar{u}_i = m_l}   x_i  &= \sum_{i |\bar{u}_i = m_l}  ( \bar{u}_i  + \lambda E_i) \nonumber \\
 \frac{ \sum_{i | \bar{u}_i = m_l}  x_i }{n_l} &=  \lambda E_i + m_l = c_l
\end{align}
from (\ref{center}). This shows that the center of a bounding ball is the mean of  the samples in the cluster. 
 
 \textbf{Part 3.} We show that the bounding balls are disjoint. Consider any two clusters $V_l$ and $V_o$ with bounding balls $C_l$ and $C_o$ centered at $c_l$ and $c_o$ with radii $r_l = \lambda(n_l -1)$ and $r_o = \lambda(n_o -1)$ respectively. 
Let $u_i \in V_l$, $u_j \in V_o$.  Then, for $p \in N, 1\leq p \le n$,
\begin{align}
   &c_l - c_o \nonumber\\
= &(m_l - m_o) + \lambda(E_i - E_j) \nonumber \\
= &(m_l - m_o) + \lambda(\sum_{p | \bar{u}_p \neq m_l} e_{ip} -  \sum_{p | \bar{u}_p \neq m_o} e_{jp}) \nonumber  \\ 
= &(m_l - m_o) + \lambda (\sum_{p | \bar{u}_p = m_o} e_{ip}  - \sum_{p | \bar{u}_p = m_l} e_{jp}  + \sum_{p |  \bar{u}_p \neq m_l \&  \bar{u}_p \neq m_o} (e_{ip} - e_{jp})) \nonumber \\
= &(m_l - m_o) + \lambda  (n_o  e_{ij} - n_l e_{ji}) +  \lambda \sum_{p |  \bar{u}_p \neq m_l \&  \bar{u}_p \neq m_o} (e_{ip} - e_{jp}) \nonumber \\
= &(m_l - m_o) + \lambda (n_l + n_o) \frac{m_l - m_o}{\|m_l - m_o\|} + \lambda\sum_{p |  \bar{u}_p \neq m_l \&  \bar{u}_p \neq m_o} (e_{ip} - e_{jp}).
\end{align}
 We now prove that  
 \begin{equation}\label{theo2part3}
 \|c_l - c_o\| \geq \|(m_l - m_o) + \lambda (n_l + n_o) \frac{m_l - m_o}{\|m_l - m_o\|}\| = \|m_l - m_o\|+\lambda (n_l + n_o) > \lambda (n_l + n_o).
 \end{equation}
 First, we prove that  $\langle m_l - m_o, e_{ip} - e_{jp}\rangle \geq 0\ \forall p | \bar{u}_p \neq m_l \&  \bar{u}_p \neq m_o$. 
 \begin{align}\label{eq:part3supp1}
&\langle m_l - m_o, e_{ip} - e_{jp} \rangle \nonumber \\
= &\langle \bar{u}_i - \bar{u}_j, e_{ip} - e_{jp} \rangle \nonumber \\
= &\langle (\bar{u}_i - \bar{u}_p )  - (\bar{u}_j - \bar{u}_p ), e_{ip} - e_{jp} \rangle \rangle \nonumber \\
= &\langle \bar{u}_i - \bar{u}_p, e_{ip} \rangle + \langle \bar{u}_j - \bar{u}_p, e_{jp} \rangle - \langle \bar{u}_i - \bar{u}_p, e_{jp} \rangle  - \langle \bar{u}_j - \bar{u}_p, e_{ij} \rangle \nonumber \\
= &\|\bar{u}_i - \bar{u}_p\| + \|\bar{u}_j - \bar{u}_p\| - \|\bar{u}_i - \bar{u}_p\| \langle e_{ip}, e_{jp} \rangle   - \|\bar{u}_j - \bar{u}_p\|  \langle e_{jp}, e_{ip} \rangle  \nonumber \\
=  &\left( \|\bar{u}_i - \bar{u}_p\| + \|\bar{u}_j - \bar{u}_p\| \right) \left(1 - \langle e_{ip}, e_{jp} \rangle\right)   \nonumber \\
\geq &0
 \end{align}
as $\langle e_{ip}, e_{jp} \rangle \leq 1$ (because $e_{ip}$ and $e_{jp}$ are unit vectors). 
 To simplify, let  $a = (m_l - m_o) + \lambda (n_l + n_o) \frac{m_l - m_o}{\|m_l - m_o\|} $ and $b_p =  e_{ip} - e_{jp}$ and $b = \sum_{p |  \bar{u}_p \neq m_l \&  \bar{u}_p \neq m_o} b_p$. Then using (\ref{eq:part3supp1}),  
 $\langle a, b_p\rangle \geq 0$, and 
 \begin{align}
 \|a + b\|^2 &= \|a\|^2 + \|b\|^2 + 2\langle a, b\rangle   \nonumber\\
 & \geq  \|a\|^2 + 2 \sum_{p |  \bar{u}_p \neq m_l \&  \bar{u}_p \neq m_o} \langle a, b_p \rangle \nonumber\\
 & \geq   \|a\|^2.
 \end{align}
Then we arrive at the first inequality of (\ref{theo2part3}), therefore, the rest follows that  $\|c_l - c_o\| > \lambda (n_l + n_o)$. With ball radii as in (\ref{radius}), 
\begin{equation}\label{eq:gap}
\|c_l - c_o\| - (r_l + r_o) >  \lambda (n_l + n_o) - (r_l + r_o) >2\lambda. 
\end{equation}
The distance between the ball centers are longer than total radii, meaning that the two balls are disjoint, apart by a distance of more than $2\lambda$ (\ref{eq:gap}).   
\end{proof}

\begin{theorem}
The bounding balls of the clusters are tight in general,  in the sense that it is possible to have an example  ($x_i \in V_l$) that stays right on the boundary of the ball: $\|x_i - c_l \| = r_l$ for any cluster. 
\end{theorem}


\textbf{Notation.} Let $v_i = c_l -x_i$ be the vector from a sample to its cluster center, then $v_i = \bar{u}_i - x_i + \lambda E_i$.

\begin{proof} The idea of this proof is to setup one case that there is a $x_p$ on the boundary of the  bounding ball.   For a fixed sample $x_p \in V_l$, let $v_p = r_l v_0$ and for any $j \neq p, x_j \in V_l$, $ v_j  = -\lambda v_0 $  for any $v_0 \in \mathbb{R}^d, \|v_0\| = 1$, making $\sum_{i | x_i \in V_l} v_i = 0$ and $\|v_p\| = r_l$, or $x_p$ is on the boundary of $C_l$. We will show that  $\partial^{\epsilon}_{\bar{u}} (f) \geq 0$ for any $\epsilon$, i.e. the samples we take can be one of the datasets with the same solutions. We show this for cluster $l$ (without loss of generality). Recall the directional derivative (\ref{eqlemma1}) with $\epsilon_i = 0\ \forall x_i \notin V_l$:

\begin{align}
\partial^{\epsilon}_{\bar{u}} (f) = &\sum_{i | x_i \in V_l} \langle v_i , \epsilon_i \rangle + \lambda \sum_{j < i | \bar{u}_i = \bar{u}_j} \|\epsilon_i - \epsilon_j\| \nonumber \\
= &\langle v_p, \epsilon_p \rangle + \sum_{j | j\neq p, x_j \in V_l} \langle v_j , \epsilon_j \rangle +  \lambda \sum_{j < i | \bar{u}_i = \bar{u}_j} \|\epsilon_i - \epsilon_j\| \nonumber \\
= &\langle \lambda (n_l -1) v_0, \epsilon_p \rangle +  \sum_{j | j\neq p, x_j \in V_l} \langle -\lambda v_0, \epsilon_j \rangle +  \lambda \sum_{j < i | \bar{u}_i = \bar{u}_j} \|\epsilon_i - \epsilon_j\| \nonumber \\
=  &\lambda \sum_{j | j\neq p, x_j \in V_l} \langle v_0, \epsilon_p - \epsilon_j \rangle +  \lambda \sum_{j < i | \bar{u}_i = \bar{u}_j} \|\epsilon_i - \epsilon_j\| \nonumber \\
\geq &\sum_{j | j\neq p, x_j \in V_l} \langle v_0, \epsilon_p - \epsilon_j \rangle  +  \lambda \sum_{j | j\neq p, x_j \in V_l} \|\epsilon_p - \epsilon_j\| \nonumber \\
= &\sum_{j | j\neq i, x_j \in V_l}  (\langle v_0, \epsilon_p - \epsilon_j \rangle + \|\epsilon_p - \epsilon_j\|) \nonumber \\
\geq & 0\ \forall \epsilon | \epsilon_i = 0\ \forall x_i \notin V_.
 \end{align}
due to $\|v_0\| =1$, $\langle v_0, \epsilon_p - \epsilon_j \rangle \geq - \|\epsilon_p - \epsilon_j\|$. 
This makes our choice of $x_p,x_j \in V_l$ a valid sample set of the cluster that have nonnegative directional derivative, resulting in the same solutions. Hence, it is possible for any cluster to have a sample that lies in the ball's boundary, making all the balls tight (cannot be any smaller).  
\end{proof}


\section{General Characteristics}
In this section, we show general characteristics on  samples that have the same solutions by convex clustering, properties and intuitive guidelines on hyperparameter settings ,  a case of impossibility of convex clustering and a guideline on  statistical consistency.

\subsection{Datasets with the same solution.} 
Consider a fixed solution of (\ref{objective}) $\bar{u}$, we investigate   all  datasets (data vectors $x$) that take $\bar{u}$ as their solutions. By the way $E_i$, $c_l$, $m_l$, $n_l$, $r_l$ and $k$ are defined, they are independent of $x$ given $\bar{u}$ and can be computed from $\bar{u}$.  Let $v = [v_1^T, v_2^T \cdots v_n^T]^T \in \mathbb{R}^{nd}$ with $v_i = c_l-x_i$.  We wish to determine the set of $x$ (equivalently, $v$)  that results in the same solution $\bar{u}$. We show a concrete formulation as follows. Let $g(\epsilon) = \lambda \sum_{i<j | \bar{u}_i = \bar{u}_j} \|\epsilon_i - \epsilon_j\|$. Let $g^*(v)$ be the Fenchel conjugate of $g(\epsilon)$.

\begin{theorem} \textbf{(General condition)}
Suppose that $\bar{u}$ is given. The necessary and sufficient condition for $x$ to arrive at solution $\bar{u}$ is that $g^*(v) \leq 0$.
\end{theorem}
As $g$ is a convex function, we know that $g^*$ is also convex and the set of $v$ resulting in the same solutions is also convex. 
\begin{proof}
The necessary and sufficient condition for $v$ to have the same solutions  is that  $\partial^{\epsilon}_{\bar{u}} (f) \geq 0$ (noting that $\|\epsilon\|$ is not necessary to be of length $1$ for the optimality condition), or
\begin{align}
\sum_i \langle v_i, \epsilon_i \rangle + \lambda \sum_{i<j| \bar{u}_i = \bar{u}_j} \|\epsilon_i-\epsilon_j\| &\geq 0 \ \forall \epsilon \nonumber \\
\langle v,\epsilon\rangle + g(\epsilon) &\geq 0\ \forall \epsilon  \nonumber \\ 
\sup_{-\epsilon}(\langle v,-\epsilon\rangle - g(-\epsilon)) &\leq 0  \nonumber \\
 g^*(v) &\leq 0.
\end{align}
\end{proof}
The merit of this theorem is that, given a dataset with its solution $\bar{u}$, we can check and characterize all other datasets having the same solution.

\subsection{What is the range of $\lambda$ for nontrivial solutions?} 
We define nontrivial solutions in the sense that $1 < k < n$. Setting $\lambda$ is not a trivial task, previous work \citep{Panahi17} has pointed out a range of $\lambda$, but not easy to compute. We show  intuitive bounds that can be easily computed from the dataset using the main result. 
\begin{theorem} Following are the necessary conditions for $\lambda$ to obtain nontrivial solutions. (1) Upper bound: for $1<k$, it is necessary that  
\begin{equation}\label{lambdabound1}
\lambda <  \max_{i\neq j} \frac{\|x_i -x_j\|}{ 2}.
\end{equation}
(2) Lower bound: for $k<n$, let $q$ be the size of the largest cluster, it is necessary that 
\begin{equation}\label{lambdabound2}
\min_{i \neq j} \frac{\|x_i - x_j\|}{\sqrt{2 q(q-1)}} \leq  \lambda.
\end{equation}

\end{theorem}

The theorem, even though not being sufficient conditions, serves as a guideline for setting $\lambda$ to obtain nontrivial solutions. We learnt that $\lambda$ should scale linearly with the magnitude of the data, and of the same magnitude as  pairwise distances of the samples. This is more intuitive than previous guidelines in \citep{Panahi17}. While $q$ is unknown, we know that $q<n$ in nontrival solutions. 

\begin{proof}

\textbf{Upper bound.} As $1<k$, there exist two samples $x_o$ and $x_l$ that belong to two different clusters. Therefore, their distance must be equal or larger than the gap between clusters. Therefore, we have the upper bound of $\lambda$:
\begin{equation}
2\lambda < \|x_o - x_l\| \leq \max_{i\neq j}   \|x_i - x_j\|. 
\end{equation}

\textbf{Lower bound.}  As $k<n$, let $V_l$ be the largest  cluster with size $q$ ($q > 1$). 
\begin{align*}
\sum_{i,j| x_i,x_j \in V_l} \|x_i - x_j\|^2 &= \sum_{i,j| x_i,x_j \in V_l}  \|(x_i - c_l) - (x_j - c_l)\|^2 \\
&= (q-1) \sum_{i | x_i \in V_l} \|x_i - c_l\|^2 -2 \sum_{i\neq j| x_i,x_j \in V_l} \langle x_i - c_l, x_j - c_l \rangle \\
&= q \sum_{i | x_i \in V_l} \|x_i - c_l\|^2 - \|\sum_{i | x_i \in V_l} (x_i - c_l) \|^2\\
&= q \sum_{i | x_i \in V_l} \|x_i - c_l\|^2\\
&\leq q^2 r_l^2 = q^2 (q-1)^2\lambda^2.
\end{align*}
As $\frac{q(q-1)}{2} \min_{i \neq j} \|x_i - x_j\|^2 \leq  \sum_{i,j| x_i,x_j \in V_l} \|x_i - x_j\|^2 \leq q^2 (q-1)^2\lambda^2$, then, 
\begin{align}\label{eqlowerbound}
\min_{i \neq j} \|x_i - x_j\|^2 &\leq 2 q(q-1) \lambda^2. \nonumber \\
\min_{i \neq j} \frac{\|x_i - x_j\|}{\sqrt{2 q(q-1)}} &\leq  \lambda.
\end{align}
\end{proof}
For example, if the largest cluster is of size 2, then $\min_{i \neq j} \frac{\|x_i - x_j\|}{2} \leq  \lambda$. Larger clusters can appear at smaller $\lambda$. In principle, $\min_{i \neq j} \frac{\|x_i - x_j\|}{\sqrt{2 (n-1)(n-2)}} (\leq \min_{i \neq j} \frac{\|x_i - x_j\|}{\sqrt{2 q(q-1)}}) \leq  \lambda$ is an absolute lower bound for nontrivial solution. However,
if we only look for clustering solutions with not too large clusters ($q \ll n$), (\ref{eqlowerbound}) can be used as the lowerbound of $\lambda$.

\subsection{Impossible to Cluster} Previous subsection is about necessary conditions of $\lambda$ for nontrivial solutions. Do we have sufficient conditions for $\lambda$ to obtain nontrivial solutions? We show that the answer is no in general. 

We show an example of collinear samples that it is not possible to find nontrivial solutions using convex clustering.  Let $x_i = i \in \mathbb{R}$, or more general, $X$ is a dataset with samples lying in a straight line with  successive samples having distance $1$. If there is a nontrivial cluster, say $V_l$ with $|V_l| = n_l$, then $\lambda < 0.5$ for bounding ball gaps. For bounding ball radii condition, given that diameter of $V_l$ is at least $n_l-1$, which is not greater than the diameter of the bounding ball, $2\lambda(n_l - 1) \geq n_l-1$, or $\lambda \geq 0.5$. These conditions on $\lambda$ are contradictory to each other. In other words, there is no $\lambda$ for convex clustering to find nontrivial solutions in this example. This is similar to agglomerative clustering.

\subsection{Consistency of Bounding Balls} Whether the balls are consistent if we sample more points ($n \to \infty$) from the same distribution? In principle, samples will fills up the support (the region with nonzero density) of the distribution. Therefore, as $n \to \infty$, if nontrivial solutions exist and with suitable $\lambda$, the solutions will converge to the bounding balls of the continuous regions of the support of the distributions. If the support of the distribution is connected, then the convergence of the solutions will be only one clusters (with suitable $\lambda$). There is no guaranteed convergences for not suitable $\lambda$. 

\begin{figure}[t]
  \centering
  \includegraphics[width=1\linewidth]{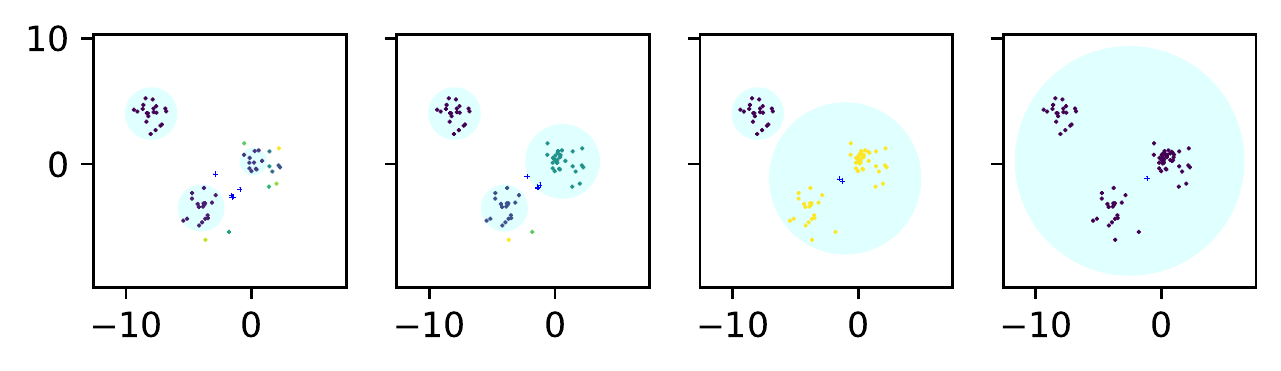}\\
  \includegraphics[width=1\linewidth]{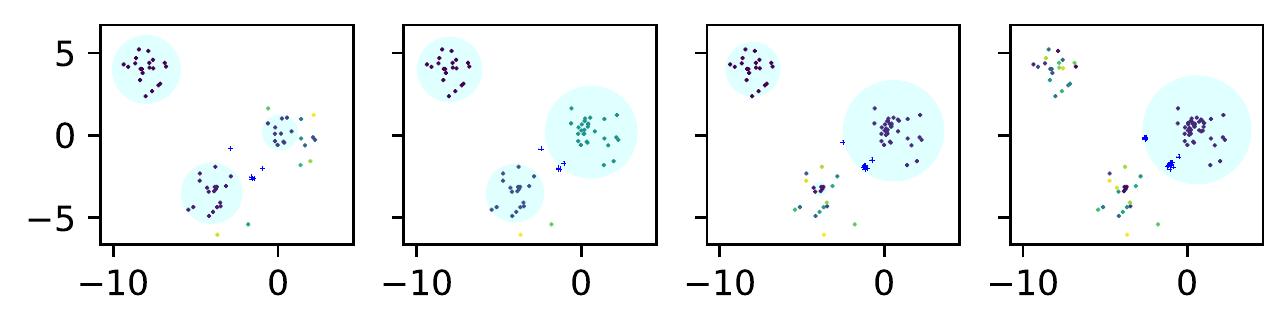}
  \caption{Solutions of convex clustering on three Gaussian components data with more samples added to one component. From left to right, there are 0, 8, 16, 24 more samples added to the interior of the right most component (all with $\lambda = 0.11$). Upper figure, with fixed $\lambda = 0.11$, the cluster corresponding to the increasing component became larger as more samples added, merging with other clusters. Lower figure: by reducing $\lambda$ ($\lambda = 0.11, 0.099, 0.088$ and $0.077$), the increasing component can be learnt as one cluster at the expense of the other components, which were split into smaller ones.}
  \label{boundingball}
\end{figure}

\section{Experiments}

We ran experiments and visualized the results to verify properties of convex clustering  (implemented with ADMM algorithm). Samples in the same cluster were plotted with the same color. For convex clustering, cluster prototypes are blue \textcolor{blue}{+}, and corresponding bounding balls has light cyan color.

\subsection{Inflexibility of convex clustering} The fact that the radii of the bounding balls are proportional to the cluster sizes can be a disadvantage of convex clustering. The case of data with well-separated clusters might not be learnt by convex clustering if the cluster sizes are not proportional to the areas of the clusters. We shows a demonstration of this case in Figure \ref{boundingball}.  We started with three Gaussian components (the left most plot) with 60 samples, then adding 8, 16 and 24 more samples to the interior of the right most component to make three more datasets (the other components remained fixed). We first ran convex clustering with fixed $\lambda = 0.11$ for all four datasets and plotted the solutions (upper figure). We found that the dataset can be clustered by convex clustering. However, adding more data to the interior of the cluster (in the subfigures on the right), the cluster corresponding to the increasing components become larger and larger (due to the increase in cluster size) to the extend that it merges with the other components. To avoid this phenomenon, one can reduce $\lambda$ to keep this component a cluster. In the  lower figure, we ran convex clustering with $\lambda = 0.11, 0.099, 0.088$ and $0.077$. However, this will break the other clusters due to smaller and smaller $\lambda$. This shows the inflexibility of convex clustering that not only the clusters must be well separated, their cardinalities must match with their areas.  

\begin{figure}
  \centering
  \includegraphics[width=0.85\linewidth]{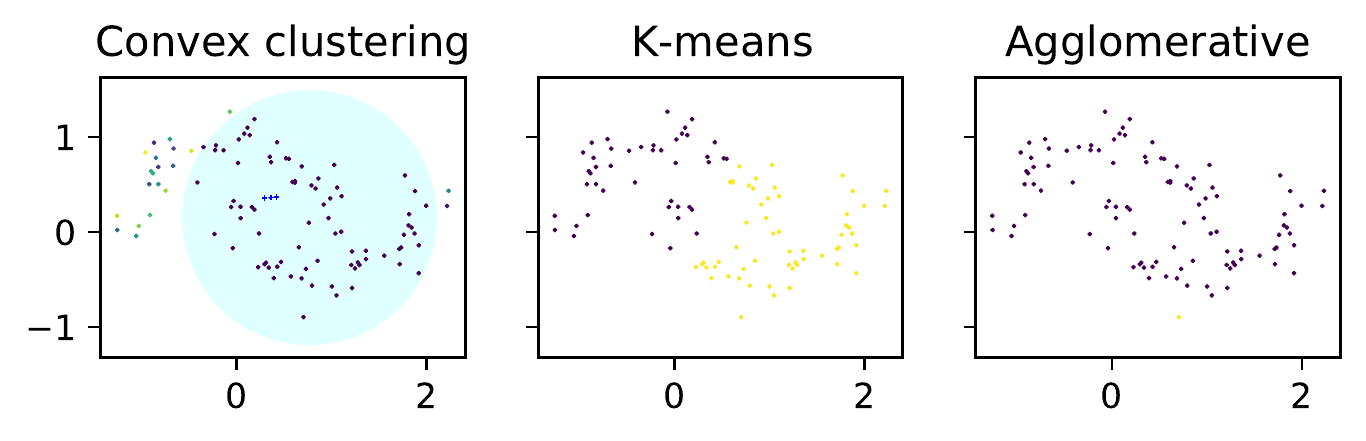}\\
    \includegraphics[width=0.85\linewidth]{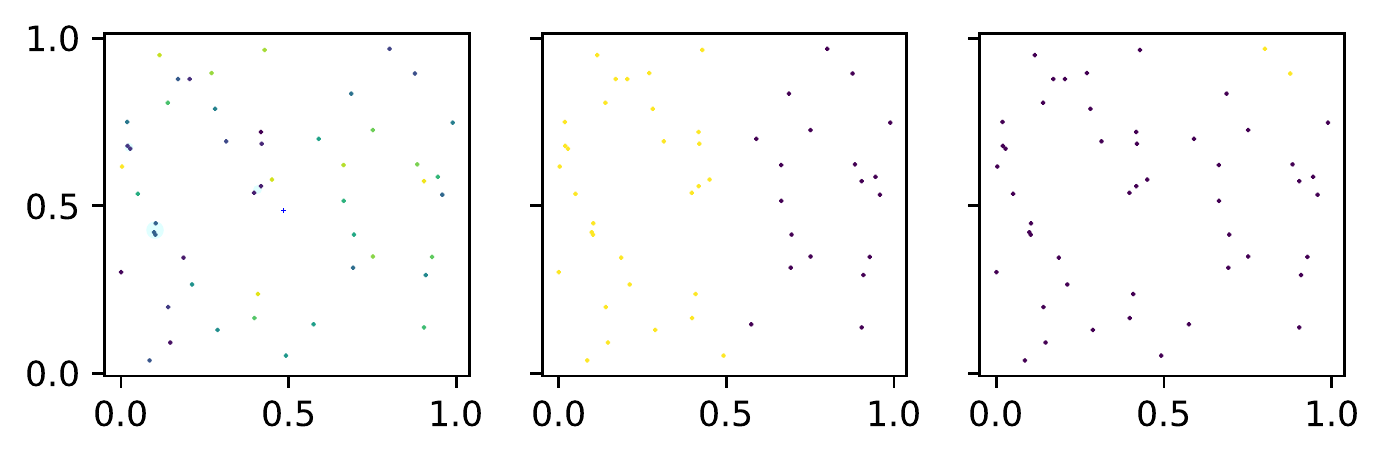}\\
        \includegraphics[width=0.85\linewidth]{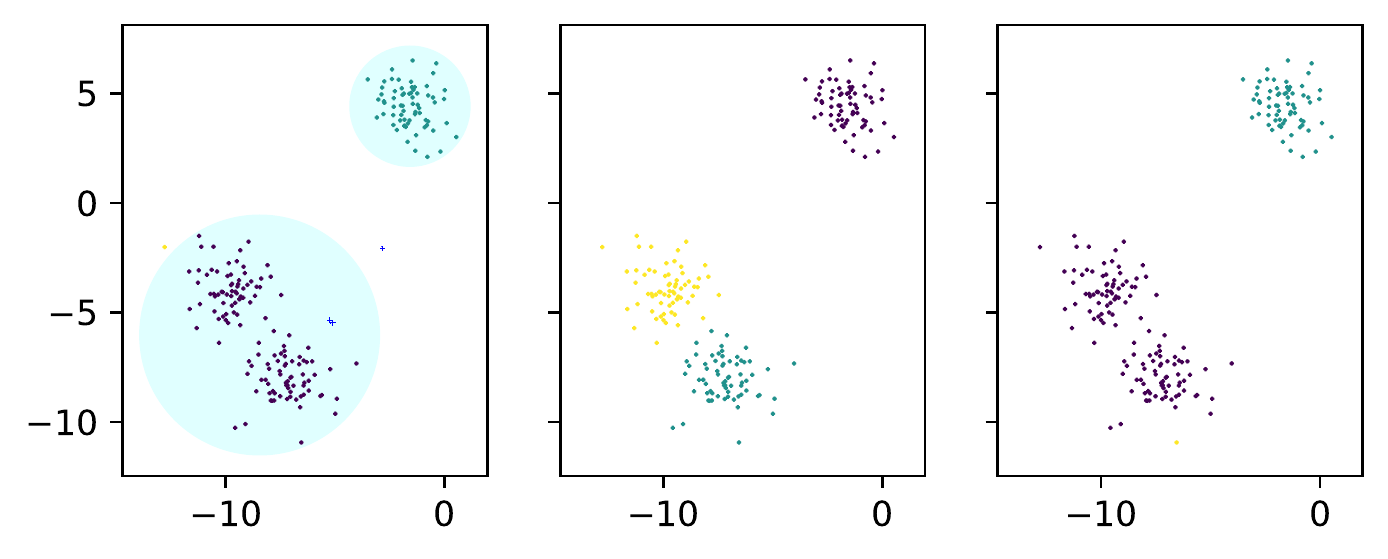}\\
    \includegraphics[width=0.85\linewidth]{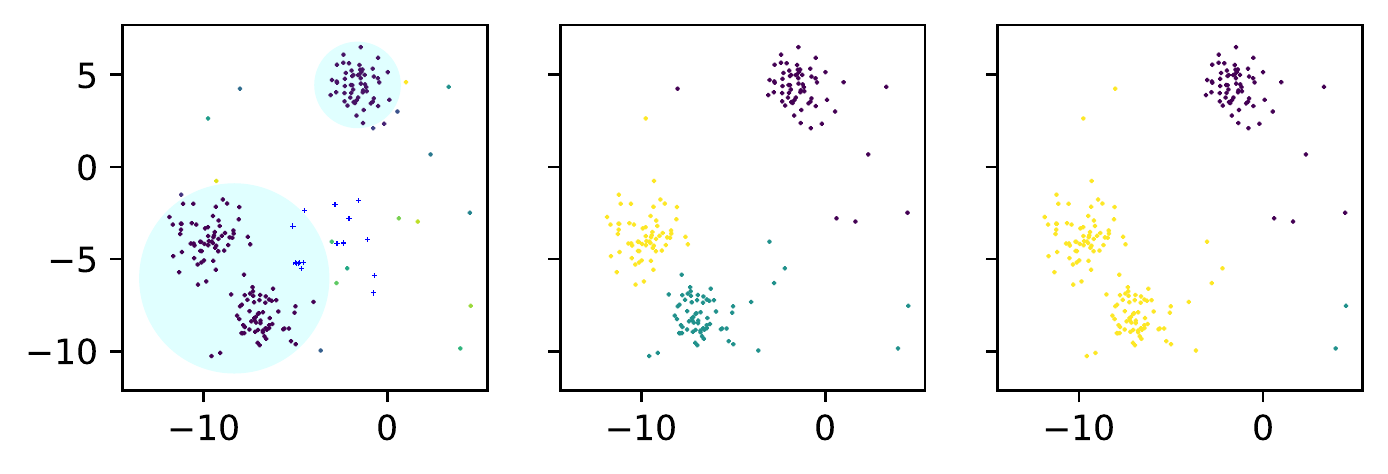}
  \caption{Clustering results of three algorithms (convex clustering, k-means clustering and agglomerative clustering) on two-moons, uniform, three-blobs  and three-blobs with noise datasets. For convex clustering, bounding balls have the shaded light cyan color. Blue dots are prototypes. Hyperparameters of the four datasets: $\lambda=0.0172,\ 0.01307,\ 0.042,\ 0.042$, selected as the cases with nontrivial large clusters.}
  \label{figmainexp}
\end{figure}%

\subsection{Difference from k-means and agglomerative clusterings}
The second experiment was to visualize some typical cases to show the difference of convex clustering from k-means and agglomerative clustering (ward linkage) with synthetic data  in Figure \ref{figmainexp}. The datasets were (row by row) \textbf{1)} two-moons  ($n=100$, noise level: $0.15$),  \textbf{2)} uniform data ($n=50$), \textbf{3)} three Gaussian components ($n=200$),  and \textbf{4)} three Gaussian components with 10$\%$ random noises.  We observed that convex clustering could find large clusters  sometimes, but did not necessarily conform to the true data distributions (no moon shapes or two-components clusters). The last dataset showed that convex clustering could ignore noisy samples far away from dense areas, suggesting  noise removal ability.

\subsection{Sensitivity to $\lambda$}
We show the solutions of convex clustering on the datasets on the paper with different $\lambda$ in Figure \ref{figsupp}. We found that in difficult cases, the solutions changed drastically with minor changes in the hyperparameter. This means that it is not easy to choose hyperparameter $\lambda$ for some desirable solutions as they are sensitive to $\lambda$. 
\begin{figure}
  \centering
     \includegraphics[width=0.85\linewidth]{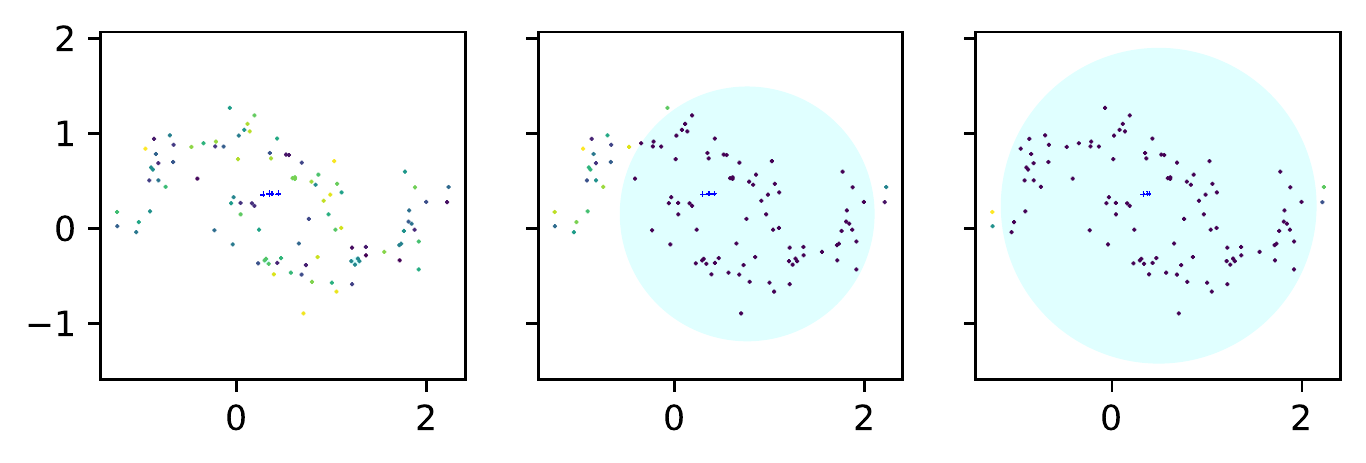}\\
     \includegraphics[width=0.85\linewidth]{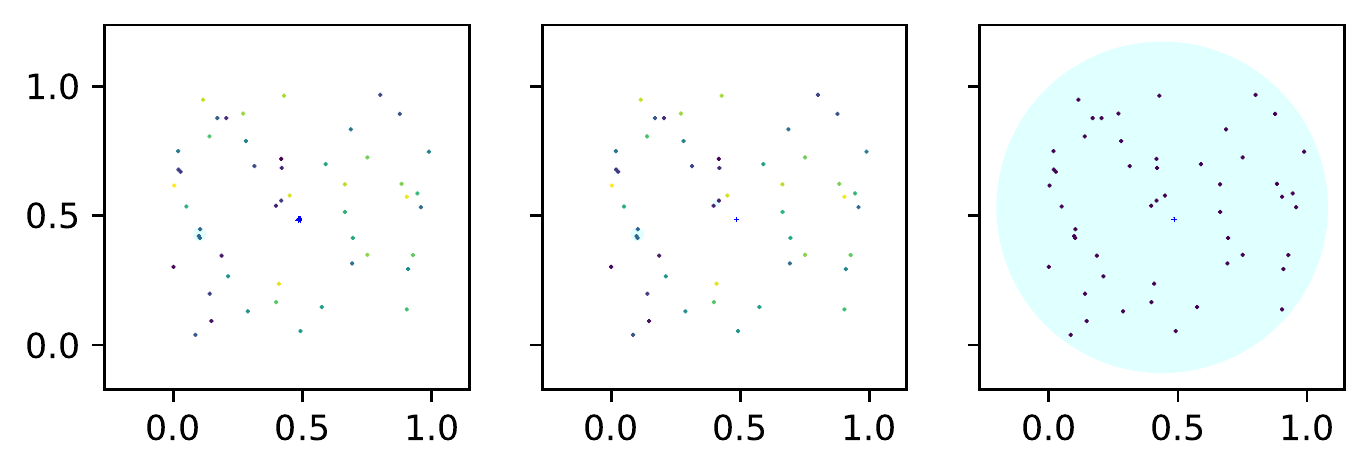}\\
     \includegraphics[width=0.85\linewidth]{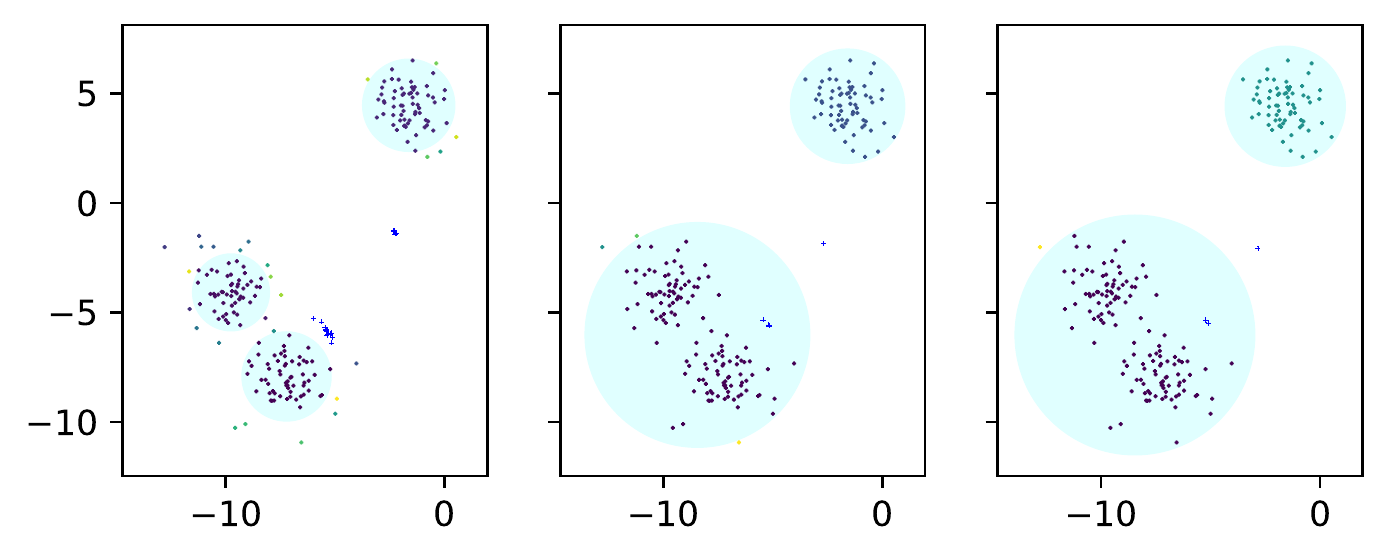}\\
     \includegraphics[width=0.85\linewidth]{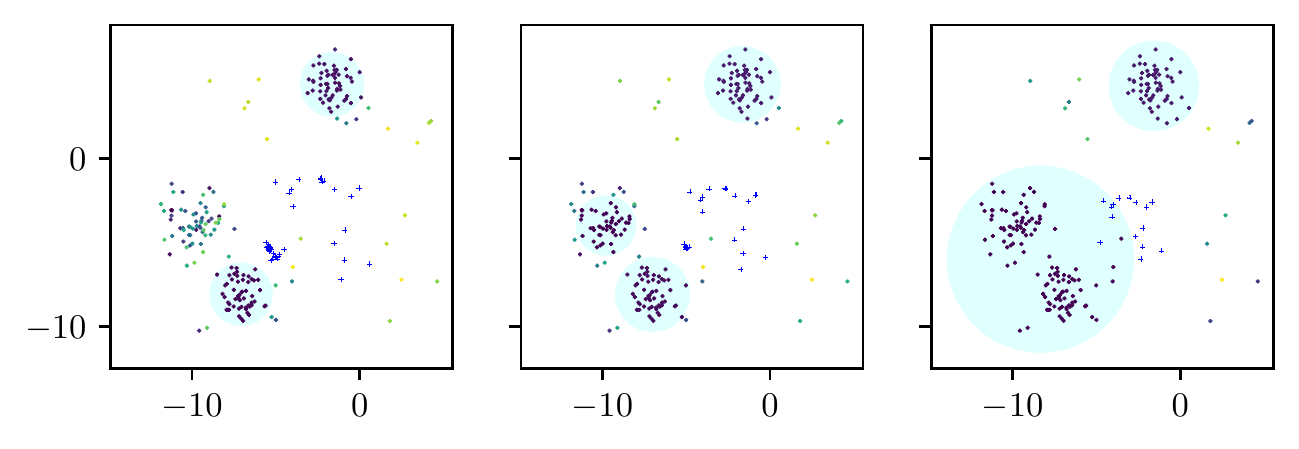}
  \caption{We generate some typical datasets (from upper to lower)  two-moons, uniform, three components, and  three components with noises datasets. Visualization of convex clustering solutions with different $\lambda$: two moons, $\lambda = .017,\ .0172,\ .0175$, uniform $\lambda = .013,\ .01307, \ .0131$, and the remaining two datasets: $\lambda = .035,\ .04,\ .045$.}
  \label{figsupp}
\end{figure}%

\section{Conclusions}
We have studied the solutions of (unweighted) convex clustering to clarify the myth of its relationship with k-means clustering and agglomerative clustering and found the following three facts. 
\begin{itemize}
\item Convex clustering produces convex clusters, like k-means clustering but unlike agglomerative clustering.
\item Unlike k-means clustering that produces clusters as Voronoi cells, clusters of convex clustering are bounded by disjoint bounding balls. Importantly, the bounding balls have significant gaps between them. We'd state that convex clustering only learns \emph{circular clusters}. 
\item Unlike k-means clustering that cluster boundary has \emph{the same distance} to closest cluster centers, the radius of a bounding ball, distance from its center to its boundary, is \emph{proportional to the size of the cluster}.  
\end{itemize}
We have also characterized  all possible datasets that produce the same solutions. We have shown an intuitive guideline for $\lambda$ to produce nontrivial clusters. We have also shown an impossible case and condition for consistency of bounding balls. We have further shown, through demonstration,  behaviors of convex clustering, and found its noise-removal ability. 

Interesting future work includes:  1) setting weights for pairwise fusion penalties. This seems to be the only way to obtain nonconvex clusters. It is expected that the weight function determines the shapes of clusters, \emph{how nonconvex} they can be. 2) Closing the gap between our insights (necessary conditions) and clustering consistency (sufficient conditions) to completely understand the behavior of convex clustering.

\vskip 0.2in

\bibliography{lrclust}

\end{document}